\documentclass[12pt]{article}
\usepackage{arxiv}
\usepackage[utf8]{inputenc} 
\usepackage[T1]{fontenc}    
\usepackage{graphicx}
\usepackage{amsmath,amssymb,amsthm}
\usepackage{mathtools}
\usepackage{xspace}
\usepackage{epsfig}
\usepackage[vlined,linesnumbered,ruled,titlenotnumbered]
{algorithm2e}
\usepackage{color}
\usepackage{url}
\usepackage{tikz}
\usepackage{float}
\usepackage{booktabs}
\usepackage{makecell}
\usepackage{multicol}
\usepackage{multirow}
\usepackage{longtable}
\usepackage{array}
\usepackage{anyfontsize}
\usepackage{balance}
\usepackage{algpseudocode}
\usepackage{subcaption}
\usepackage{nicefrac}
\usepackage{microtype}      

\DeclareMathOperator*{\argmax}{\arg\!\max}
\DeclareMathOperator*{\argmin}{\arg\!\min}

\newtheorem{thm}{Theorem}

\newtheorem{exm}{Example}

\newcommand{\ignore}[1]{}

\begin{document}

\author{Anh Viet Do
\\Optimisation and Logistics
\\The University of Adelaide, Adelaide, Australia
\And
Jakob Bossek
\\Optimisation and Logistics
\\The University of Adelaide, Adelaide, Australia
\And
Aneta Neumann
\\Optimisation and Logistics
\\The University of Adelaide, Adelaide, Australia
\And
Frank Neumann
\\Optimisation and Logistics
\\The University of Adelaide, Adelaide, Australia
}
\title{Evolving Diverse Sets of Tours for the Travelling Salesperson Problem}

\maketitle
\begin{abstract}
Evolving diverse sets of high quality solutions has gained increasing interest in the evolutionary computation literature in recent years. With this paper, we contribute to this area of research by examining evolutionary diversity optimisation approaches for the classical Traveling Salesperson Problem (TSP). We study the impact of using different diversity measures for a given set of tours and the ability of evolutionary algorithms to obtain a diverse set of high quality solutions when adopting these measures. Our studies show that a large variety of diverse high quality tours can be achieved by using our approaches. Furthermore, we compare our approaches in terms of theoretical properties and the final set of tours obtained by the evolutionary diversity optimisation algorithm.
\end{abstract}
\keywords{Evolutionary algorithms, diversity maximisation, travelling salesperson problem}
\section{Introduction}

Evolutionary diversity optimisation aims to construct a set of diverse solutions that all have high quality but differ with respect to important properties of the solutions. This area of research started by Ulrich and Thiele~\cite{DBLP:conf/gecco/UlrichBZ10,DBLP:conf/gecco/UlrichT11} has recently gained significant attention within the evolutionary computation community. It equips practitioners with high quality solutions of variable designs. Furthermore, the methods developed can be used in the context of machine learning and algorithm selection and configuration. Here, evolutionary diversity optimisation is used to evolve a diverse set of instances that can be used for training prediction models that forecast algorithm performance on a newly given instance.

Recent studies in this area of research mainly focused on the diversity measure to be used in the selection process when characterising solutions by underlying features. Differences in feature values as well as weighted combinations of two or more features have been examined in the context of Travelling Salesperson Problem (TSP) instances and images~\cite{DBLP:conf/ppsn/GaoNN16,DBLP:conf/gecco/AlexanderKN17}. More recently the use of the discrepancy measure as well as the use of popular indicators from the area of evolutionary multi-objective optimisation have been proposed and evaluated~\cite{DBLP:journals/corr/abs-1802-05448,DBLP:conf/gecco/NeumannG0019}. Special mutation operators for creating diverse sets of TSP instances have been investigated in~\cite{DBLP:conf/foga/BossekKN00T19}.

In this paper, we examine, for the first time, TSP in the context of evolutionary diversity optimisation. Our focus is on the diversity of the tours themselves, not their qualities. The TSP has been subject to a wide range of studies. Various evolutionary algorithms and other heuristic search methods as well as exact solvers have been developed over the years~\cite{lin1973,helsgaun2000,xie2009,Nagata2013}.
Furthermore, the TSP has been studied widely in the area of algorithm selection and configuration~(see \cite{DBLP:journals/ec/KerschkeHNT19} for a recent overview on this research area) and a wide range of studies regarding important features of TSP instances and their relation to algorithm performance have been carried out~\cite{smithmiles2011,kanda2011,mersmann2013,KKBHTLeveragingTSP}.  Evolutionary diversity optimisation has so far only been considered for evolving TSP instances~\cite{DBLP:conf/ppsn/GaoNN16,DBLP:journals/corr/abs-1802-05448,DBLP:conf/gecco/NeumannG0019}, but not for computing a diverse set of high quality tours for a given TSP instance. 

In this paper, we examine ways to evolve a diverse set of high quality TSP tours for a given TSP instance and population size.
A crucial aspect in our study is how to measure the diversity of a tours population for an effective diversity-driven approach. To establish diversity-oriented evolutionary pressure, we propose two diversity measures: an edge distribution diversity measure and a pairwise dissimilarity measure. We study both measures with respect to their properties in the context of evolutionary diversity optimisation and carry out experimental investigations showing how evolutionary algorithms generate diverse populations using the two measures on unweighted TSP instances where all tours are accepted. Our results for this basic set up shows that these cases can be handled effectively and set the basis for studies on classical TSPlib instances~\cite{DBLP:journals/informs/Reinelt91} where tours are filtered based on qualities.

We investigate the introduced evolutionary diversity optimisation approach using the two diversity measures on classical instances, and explore the differences in results when using classical $k$-opt operations where $k=2,3,4$ for a wide range of quality thresholds imposed on the tours. Furthermore, we study how the population size which determines the size of the set of tours effects the ability of the evolutionary diversity optimisation approaches to obtain a high diversity score.

The paper is structured as follows. In Section~\ref{sec2}, we introduce the Traveling Salesperson Problem in the context of evolutionary diversity optimisation and the algorithms that are part of our investigations. In Section~\ref{sec3}, we introduce the edge diversity measure and investigate its theoretical properties. Section~\ref{sec4} introduces the pairwise distance diversity measure along with its theoretical properties. We report on our experimental investigations in Section~\ref{sec5} and finish with some conclusions.

\section{Maximising Diversity in TSP}
\label{sec2}

The Traveling Salesperson problem (TSP), one of the best-known $\mathcal{NP}$-hard combinatorial optimisation problems, can be described as follows: Given a complete graph $G=(V,E)$ with $n = |V|$ cities, $m = n(n-1)/2 = |E|$ edges and the pairwise distances between the cities, the goal is to compute a tour of minimal length that visits each city exactly once and finally returns to the original city. For the Euclidean TSP all cities lie in the Euclidean plane and the pairwise distances between the cities are determined by the Euclidean metric. Let $V = \{1, \ldots, n\}$. The goal is to find a permutation $\pi : V \rightarrow V$ that minimises the cost function 
\begin{align*}
c(\pi) = d(\pi(n),\pi(1)) + \sum_{i=1}^{n-1} d(\pi(i),\pi(i+1)), 
\end{align*}
where $d(i,j)$ is the Euclidean distance between points $i$ and $j$. Note that the Euclidean TSP remains an $\mathcal{NP}$-hard combinatorial optimisation problem. 

In this paper, we consider diversity optimisation for the Traveling Salesperson Problem. For each TSP instance, our goal is to find a set $P$ of $\mu = |P|$ tours that is diverse with respect to some diversity measure, while each tour meets a given quality threshold. Let $I$ be an individual (which constitutes a permutation of the given $n$ cities) and
$c(I)$ be the cost of $I$. The quality threshold is met iff  $c(I) \leq (1 + \alpha) \cdot OPT$, where $OPT$ is the value of an optimal tour and $\alpha>0$ is a parameter that determines the required quality of a desired solution. The quality criterion means that the quality threshold is met iff $I$ is a $(1 + \alpha)$ approximation of an optimal solution. We assume that the optimal tour is known for a given TSP instance.

In order to optimise the diversity for the Traveling Salesperson Problem we employ a $(\mu+1)$-EA  that has already been used in the context of evolutionary diversity optimisation~\cite{DBLP:conf/ppsn/GaoNN16,DBLP:conf/gecco/AlexanderKN17,DBLP:journals/corr/abs-1802-05448,DBLP:conf/gecco/NeumannG0019,DBLP:conf/foga/BossekKN00T19}. Our approach differs from the previous ones in terms of the considered problem and underlying diversity measure that drives the optimisation approach.

We use Algorithm~\ref{alg:ea} to compute a diverse population consisting of TSP tours where each individual/tour $I$ has to meet a given quality criteria $c(I)$ according to a given threshold. Initially, the population $P$ is generated with $\mu$ individuals, and exactly one offspring $I'$ is produced in each iteration. If the offspring $I'$ does not satisfy $c(I') \leq (1+\alpha)\cdot OPT$, then it is discarded. Otherwise $I'$ is added to the population. Afterwards, elitist survival selection is performed with respect to a diversity measure $D$. For our investigations, we consider some function indicating overlaps between tours $D \colon P \rightarrow \mathbb{R}$, which should be minimized to improve diversity. Thus, an individual $I \in P$ is removed such that $D(P\setminus \{I\})$ is minimal among all individuals $J \in P$.

We introduce two diversity measures for evolutionary diversity optimisation for TSP, namely an \emph{edge diversity} (ED) optimisation approach in Section~\ref{sec3} and a \emph{pairwise edge distances} (PD) optimisation approach in Section~\ref{sec4}. We consider vector functions instead of traditional scalar functions, with the hope to capture more nuances in the survival selection mechanism elegantly.

The edge diversity optimisation approach attempts to maximise population diversity without relying on dissimilarities between tours. Instead, it considers how frequent each edge is present in the population, and aims to equalise these frequencies. The goal is a population of tours containing every edge, each as close to $k$ times as possible for some $k$. The approach makes sense if edges can be considered independent of each other, meaning each edge is present in the population at a frequency independent of that of other edges. This is not true for tours, thus important information may potentially be left out.

On the other hand, the pairwise edge distances approach considers solely the edge distances between all pairs of tours in the population in terms of overlap. It attempts to simultaneously maximise these distances and equalising them, with an emphasis on pairs with the least distances. In effect, it tries to increase these small distances by moving tours away from their closest tours and possibly closer to further tours, using mutation operators. Consequently, it tends to generally increase the dissimilarities between a tour $I$ and the rest $P\setminus\{I\}$. This helps lessen the clustering phenomenon, where tours in the population form low-diversity clusters.

The evolutionary diversity algorithm based on edge diversity measure is compared to the evolutionary diversity optimisation approach based on pairwise edge distances measure in two different settings, using simple unconstrained TSP tours and TSPlib instances~\cite{DBLP:journals/informs/Reinelt91} in Section~\ref{sec5}.

\begin{algorithm}[tp]
\SetKwData{Left}{left}\SetKwData{This}{this}\SetKwData{Up}{up}
\SetKwInOut{Input}{input}\SetKwInOut{Output}{output}

{Initialise the population $P$ with $\mu$ TSP tours such that $c(I) \leq   (1+ \alpha)\cdot OPT$ for all $I \in P$.\\
Choose $I \in P$ uniformly at random and produce an offspring $I'$ of $I$ by mutation.\\ 
If $c(I') \leq   (1+ \alpha)\cdot OPT$, add $I'$ to $P$. \\
If $|P| = \mu+1$, remove exactly one individual $I$, where $I=\arg \min_{J \in P} D(P\setminus \{J\})$, from $P$.\\
Repeat steps 2 to 4 until a termination criterion is reached.\\
}
\caption{Diversity maximising ($\mu + 1$)-EA}
\label{alg:ea}
\end{algorithm}
\section{Maximising edge diversity}
\label{sec3}

In this approach, we consider diversity in terms of equal representations of edges by tours in the population, or \emph{edge diversity}. It takes into account, for each edge, the number of tours containing it, among the $\mu$ solutions in the population. These numbers are referred to as \emph{edge counts}. Given a population $P$ and an edge $e \in E$, we denote by $n(e, P)$ its edge count, which is defined,
\begin{align*}
    n(e, P) := \left|\{T \in P \, | \, e \in E(T)\}\right| \in \{0, \ldots, \mu\}
\end{align*}
where $E(T) \subset E$ is the set of edges used by solution $T$. Then in order to maximise the edge diversity we aim to minimise the vector
\begin{align*}
    \mathcal{N}(P) = \text{sort}\left(n(e_1, P), n(e_2, P), \ldots, n(e_m, P)\right)
\end{align*}
in the lexicographic order where sorting is performed in descending order. This is based on the idea of maximising genotypic diversity as the mean of pairwise edge distances~\cite{DiversityPermutation}. Since $\mu$ is fixed, we change the mean to a sum to simplify the function
\begin{align*}
    gtype(P) = \sum_{T_1\in P}\sum_{T_2\in P}|E(T_1)\setminus E(T_2)|.
\end{align*}
The maximum edge distance between two different tours is $n$, so the maximum diversity is $\mu(\mu-1)n$. Let $n_i=n(e_i,P)$. There are $n_i$ tours in $P$ sharing edge $e_i$. Therefore, it affects $n_i(n_i-1)/2$ pair-wise edge distances, reducing each by $1$. Since they can be added up independently across all edges, the diversity measure is then
\begin{align*}
    gtype(P) = \mu(\mu-1)n+\sum_{i}n_i-\sum_in^2_i.
\end{align*}
Since $\sum_in_i=\mu n$ is constant, maximising diversity is reduced to minimising $\sum_in^2_i$. Given $\sum_in_i$ being constant, the Cauchy–Schwarz inequality implies that $\sum_in^2_i$ is the smallest when all $n_i$ are as close to being equal to each other as possible. The population with such property minimises $\mathcal{N}$, meaning
\begin{align*}
    \argmax_P\{gtype(P)\} = \argmin_P\{\mathcal{N}(P)\}.
\end{align*}
For complete graphs, the optima for $\mathcal{N}(P)$ can be determined based on the following result, so the maximum diversity can be calculated.
\begin{thm}\label{theo:ed_opt}
For every pair of integers $\mu\geq1$ and $n\geq3$, given a complete graph $G=(V,E)$ where $|V|=n$, there is a $\mu$-size population $P$ of tours such that
\[\max_{e\in E}\{n(e,P)\}-\min_{e\in E}\{n(e,P)\}\leq1.\]
\end{thm}
\begin{proof}
We prove this by defining a way to construct such a population. According to Theorem~1 in \cite{Alspach1990}, in every complete graph with $n\geq3$ vertices, there is a set of $h=\left\lfloor\frac{n-1}{2}\right\rfloor$ pairwise edge-disjoint Hamiltonian cycles. We denote by $H$ the set of all such Hamiltonian cycles in $G$, and $E(H)$ the set of all edges $H$ contains. We consider 2 cases: $n$ is odd, and $n$ is even.

Assuming $n$ is odd, then $G$ can be decomposed completely into edge-disjoint tours, meaning $E(H)=E$. Let $\mu=kh+r$ where $r\in[0,h)$ and $r\equiv \mu\bmod h$. We construct a population $P'$ by adding all tours in $H$, each exactly $k$ times. We then construct $P$ from $P'$ by adding all tours in $L$ for any $L\subset H$ where $|L|=r$. With this, $n(e,P)=k$ for all $e\notin E(L)$, and $n(e,P)=k+1$ for all $e\in E(L)$.

Assuming $n$ is even, then according to \cite{Alspach1990}, for any perfect matching $M$ in $G$, the sub-graph $G^*=(V,E\setminus M)$ can be decomposed completely into edge-disjoint tours. This means $E(H)=E\setminus M$ for some perfect matching $M$ in $G$. Let $T$ be the tour in $G$ that goes through all edges in $M$, $M'=E(T)\setminus M$ be another perfect matching edge-disjoint with $M$, $H'$ be the set of edges-disjoint tours in $G'=(V,E\setminus M')$, and $\mu=k(2h+1)+r$ where $r\in[0,2h]$ and $r\equiv \mu\bmod (2h+1)$. We construct $P$ with the following steps:
\begin{enumerate}
    \item Add all tours in $H$, each $k$ times
    \item Add all tours in $H'$, each $k$ times
    \item Add tour $T$ $k$ times
    \item Add all tours in $L$ where $L\subseteq H$ and $|L|=\min\{r,h\}$, add $T$ if $r>h$
    \item If $r>h+1$, add all tours in $K\subset H'$ such that $|K|=r-h-1$.
\end{enumerate}
The result is a population $P$ such that $|P|=\mu$. If $r=0$, $n(e,P)=2k$ for all $e\in E$. Otherwise, if $r\leq h$, $n(e,P)=2k+1$ for all $e\in E(L)$ and $n(e,P)=2k$ otherwise. If $r>h$, $n(e,P)=2k+2$ for all $e\in M'\cup E(K)$ and $n(e,P)=2k+1$ otherwise ($K=\emptyset$ if $r\leq h+1$).
\end{proof}
As demonstrated by the proof, a drawback of this approach is that it does not necessarily prevent duplication of tours in $P$ when $\mu$ is sufficiently large. This is because the sum of $L1$-norm distances is no more sensitive to small distances than large distances. Consequently, the diversity score based on it can still be large when the population consists of low-diversity sub-populations that are highly dissimilar. In other words, this approach is susceptible to clustering.

In the context of EAs, this approach formulates a survival selection mechanism: removing from the population the individual $I\in\argmin_{I\in P}\{\mathcal{N}(P\setminus\{I\})\}$. For an efficient implementation, we consider an equivalent fitness function for each individual $I\in P$
\[n_P(I)=\text{sort}\left((n(e,P))_{e\in I}\right)\]
with descending sorting order. The survival selection mechanism then removes $I\in\argmax_{I\in P}\{n_P(I)\}$. The equivalence can be shown. Since elements of $\mathcal{N}(P\setminus\{I\})$ and $n_P(I)$ are in descending order, they can each be uniquely defined by a vector $(m^I_i)_{i=1,\dots,\mu}$ and $(n^I_j)_{j=1,\dots,\mu}$, respectively, where $m^I_i$ and $n^I_i$ are numbers of elements in $\mathcal{N}(P\setminus\{I\})$ and $n_P(I)$ equal to $i$, respectively. For $X,Y\in P$, if $\mathcal{N}(P\setminus\{X\})<\mathcal{N}(P\setminus\{Y\})$ lexicographically, then there must be $j\in[1,\mu]$ such that $m^X_j<m^Y_j$ and $m^X_i=m^Y_i$ for all $i\in(j,\mu]$. Fewer elements equal to $j$ in $\mathcal{N}(P\setminus\{X\})$ must be the consequence of removing more of them, meaning $n^X_j>n^Y_j$ and $n^X_i=n^Y_i$ for all $i\in(j,\mu]$. This implies that $n_P(X)>n_P(Y)$ lexicographically. On the other hand, it is trivial to prove that $\mathcal{N}(P\setminus\{X\})=\mathcal{N}(P\setminus\{Y\})$ implies $n_P(X)=n_P(Y)$. This means
\[\argmin_{I\in P}\{\mathcal{N}(P\setminus\{I\})\}=\argmax_{I\in P}\{n_P(I)\}.\]
With this method, the survival selection mechanism is consisted of three phases:
\begin{enumerate}
    \item Calculating the edge counts table: $O(\mu n)$
    \item Calculating $n_P(I)$ for all $I$: $O(\mu n\log n)$
    \item Finding and removing $I=\argmax_{I\in P}\{n_P(I)\}$: $O(\mu n)$
\end{enumerate}
The complexity of the survival selection is then $O(\mu n\log n)$. On the other hand, the mechanism's complexity when using $gtype(P\setminus\{I\})$ as fitness values is $O(\mu^2n+\mu^3)$ ($O(\mu^2n)$ from calculating the edge distances table and $O(\mu^2)$ from calculating $gtype$ from said table for each tour). Asymptotically speaking, our proposal can be faster per iteration in many cases.

\section{Equalising pairwise edge distances}
\label{sec4}

One way to remedy the clustering phenomenon is to increase edge distances between highly similar tours, potentially at the cost of decreasing edge distances between highly dissimilar tours. This method essentially equalises the tours' pairwise edge distances. As such, we present another approach that discourages clustering by emphasising uniform pairwise edge distances while maximising diversity. We reformulate the edge distance as edge overlap
\begin{align*}
    o_{XY}=|E(X)\cap E(Y)|=n-|E(X)\setminus E(Y)|,\forall X,Y\in P.
\end{align*}
The aim is then to minimise the vector
\begin{align*}
    \mathcal{D}(P) = \text{sort}\left(\left(o_{XY}\right)_{X,Y\in P}\right)
\end{align*}
in the lexicographic order where sorting is performed in descending order. This approach simultaneously maximises diversity via minimising $\sum_{X,Y\in P}o_{XY}$, while also maximising uniformity via equalising $o_{XY}$. It can be the case that $\mathcal{D}(P)<\mathcal{D}(P')$ and $gtype(P)<gtype(P')$.
\begin{exm}
Consider a problem instance with a complete graph of $n=5$ vertices and $\mu=3$. Let there be four tours $T_1=(1,3,5,4,2)$, $T_2=(1,5,4,3,2)$, $T_3=(1,2,5,3,4)$ and $T_4=(1,5,2,3,4)$, and two populations $P_1=\{T_1,T_2,T_3\}$ and $P_2=\{T_1,T_2,T_4\}$, we have
\begin{itemize}
    \item $gtype(P_1)=18<gtype(P_2)=20$
    \item $\mathcal{D}(P_1)=(2,2,2)<\mathcal{D}(P_2)=(3,2,0).$
\end{itemize}
It is clear that $P_1$ maximises edge distances uniformity and $P_2$ maximises $gtype$. Furthermore, we can see that among all $3$-size populations of maximum $gtype$, $P_2$ maximises edge distances uniformity. Likewise, among all $3$-size populations of maximum edge distances uniformity, $P_1$ maximises $gtype$. This example shows that maximising $gtype$ and edge distances uniformity at the same time can be non-trivial.
\end{exm}
In such cases, it is likely that $o_{XY}$ are more equalised in $P$ than in $P'$. Moreover, the more $gtype(P)$ is larger than $gtype(P')$, the more likely that $\mathcal{D}(P)<\mathcal{D}(P')$. This means evolving with lowering $\mathcal{D}(P)$ as the evolutionary force tends to increase the $gtype$ score. However, the maximisation performance is potentially compromised by the emphasis on edge distances uniformity, so the $gtype$ score can decrease and its maximum is sometimes unreachable. Given the drawback of this diversity measure, this may not be undesirable.

\begin{table*}[t]
\renewcommand{\arraystretch}{1.2845}
\renewcommand\tabcolsep{3.1pt}
\caption{Comparison in terms of diversity (gtype), the number of iterations (\#iter) and results of statistical testing (stat) between variants with ED approach in unconstrained cases. The Kruskal-Wallis test and the Bonferroni correction method~\cite{Corder09} are used on \#iter. $X^+$ means the measure is larger than the one for variant $X$, $X^-$ means smaller and $X^*$ means no difference.
}
\label{tb:Edge_counts}
\begin{scriptsize}
\begin{tabular}{@{}rrccrrcccrrcccrrc}
    \toprule

\multirow{2}{*}{$n$} & \multirow{2}{*}{$\mu$} &
\multicolumn{5}{c}{\bfseries 2-OPT(1)} & \multicolumn{5}{c}{\bfseries 3-OPT(2)} & \multicolumn{5}{c}{\bfseries 4-OPT(3)} \\
\cmidrule(l{2pt}r{2pt}){3-7} \cmidrule(l{2pt}r{2pt}){8-12} \cmidrule(l{2pt}r{2pt}){13-17}
 &  & \textbf{gtype} & \textbf{std} & \textbf{\# iter} &\textbf{std} & \textbf{stat}
&\textbf{gtype} & \textbf{std} & \textbf{\# iter} &\textbf{std} & \textbf{stat}
 & \textbf{gtype} & \textbf{std} & \textbf{\# iter} &\textbf{std} & \textbf{stat}\\
\midrule

\multirow{4}{*}{50}&3&100.00\%&0.00&104.50&59.19&$2^{+},3^{+}$&100.00\%&0.00&68.60&35.67&$1^{-},3^{*}$&100.00\%&0.00&\textbf{54.37}&24.28&$1^{-},2^{*}$\\
&10&100.00\%&0.00&\textbf{1635.80}&485.88&$2^{*},3^{*}$&100.00\%&0.00&1808.43&484.96&$1^{*},3^{*}$&100.00\%&0.00&1736.60&545.36&$1^{*},2^{*}$\\
&20&100.00\%&0.00&\textbf{25383.03}&8594.81&$2^{-},3^{-}$&99.97\%&0.02&49881.60&471.00&$1^{+},3^{*}$&99.93\%&0.02&50000.00&0.00&$1^{+},2^{*}$\\
&50&99.99\%&0.00&\textbf{125000.00}&0.00&$2^{*},3^{*}$&99.95\%&0.00&\textbf{125000.00}&0.00&$1^{*},3^{*}$&99.93\%&0.01&\textbf{125000.00}&0.00&$1^{*},2^{*}$\\
\hline
\multirow{4}{*}{100}&3&100.00\%&0.00&170.83&98.50&$2^{*},3^{+}$&100.00\%&0.00&137.53&69.76&$1^{*},3^{*}$&100.00\%&0.00&\textbf{102.13}&58.69&$1^{-},2^{*}$\\
&10&100.00\%&0.00&1901.73&632.60&$2^{*},3^{+}$&100.00\%&0.00&1598.80&386.85&$1^{*},3^{+}$&100.00\%&0.00&\textbf{1305.63}&343.38&$1^{-},2^{-}$\\
&20&100.00\%&0.00&\textbf{8422.40}&1742.96&$2^{-},3^{*}$&100.00\%&0.00&9902.63&1984.39&$1^{+},3^{*}$&100.00\%&0.00&9414.27&2018.82&$1^{*},2^{*}$\\
&50&99.95\%&0.00&\textbf{500000.00}&0.00&$2^{*},3^{*}$&99.86\%&0.00&\textbf{500000.00}&0.00&$1^{*},3^{*}$&99.82\%&0.01&\textbf{500000.00}&0.00&$1^{*},2^{*}$\\
\hline
\multirow{4}{*}{200}&3&100.00\%&0.00&401.03&213.80&$2^{+},3^{+}$&100.00\%&0.00&254.80&111.06&$1^{-},3^{*}$&100.00\%&0.00&\textbf{197.60}&78.85&$1^{-},2^{*}$\\
&10&100.00\%&0.00&3350.47&958.53&$2^{+},3^{+}$&100.00\%&0.00&2261.10&570.16&$1^{-},3^{*}$&100.00\%&0.00&\textbf{1888.87}&510.91&$1^{-},2^{*}$\\
&20&100.00\%&0.00&10189.73&2233.72&$2^{+},3^{+}$&100.00\%&0.00&8959.07&2242.66&$1^{-},3^{*}$&100.00\%&0.00&\textbf{7901.73}&1286.07&$1^{-},2^{*}$\\
&50&100.00\%&0.00&\textbf{76856.37}&14056.12&$2^{-},3^{-}$&100.00\%&0.00&102039.53&17118.06&$1^{+},3^{-}$&100.00\%&0.00&132837.73&25315.99&$1^{+},2^{+}$\\
\hline
\multirow{4}{*}{500}&3&100.00\%&0.00&974.33&509.02&$2^{+},3^{+}$&100.00\%&0.00&631.20&267.24&$1^{-},3^{*}$&100.00\%&0.00&\textbf{560.77}&332.16&$1^{-},2^{*}$\\
&10&100.00\%&0.00&6608.33&2298.31&$2^{*},3^{+}$&100.00\%&0.00&5281.23&1729.61&$1^{*},3^{+}$&100.00\%&0.00&\textbf{3537.07}&965.71&$1^{-},2^{-}$\\
&20&100.00\%&0.00&20888.53&5343.92&$2^{+},3^{+}$&100.00\%&0.00&15007.67&3265.35&$1^{-},3^{+}$&100.00\%&0.00&\textbf{11821.93}&3388.06&$1^{-},2^{-}$\\
&50&100.00\%&0.00&93280.70&17697.95&$2^{+},3^{+}$&100.00\%&0.00&67817.73&13010.55&$1^{-},3^{*}$&100.00\%&0.00&\textbf{64425.13}&9267.66&$1^{-},2^{*}$\\
\hline
\end{tabular}
\end{scriptsize}
\end{table*}

\begin{table*}[htbp]
\renewcommand{\arraystretch}{1.2845}
\renewcommand\tabcolsep{3.1pt}
\caption{Comparison in terms of diversity (gtype), the number of iterations (\#iter) and results of statistical testing (stat) between variants with PD approach in unconstrained cases. Tests and notations used are the same as in Table~\ref{tb:Edge_counts}.}
\label{tb:Pair_dist}
\begin{scriptsize}
\begin{tabular}{@{}rrccrrcccrrcccrrc}
    \toprule
\multirow{2}{*}{$n$} & \multirow{2}{*}{$\mu$} &
\multicolumn{5}{c}{\bfseries 2-OPT(1)} & \multicolumn{5}{c}{\bfseries 3-OPT(2)} & \multicolumn{5}{c}{\bfseries 4-OPT(3)} \\
\cmidrule(l{2pt}r{2pt}){3-7} \cmidrule(l{2pt}r{2pt}){8-12} \cmidrule(l{2pt}r{2pt}){13-17}
 &  & \textbf{gtype} & \textbf{std} & \textbf{\# iter} &\textbf{std} & \textbf{stat}
&\textbf{gtype} & \textbf{std} & \textbf{\# iter} &\textbf{std} & \textbf{stat}
 & \textbf{gtype} & \textbf{std} & \textbf{\# iter} &\textbf{std} & \textbf{stat}\\
\midrule

\multirow{4}{*}{50}&3&100.00\%&0.00&83.43&42.67&$2^{*},3^{+}$&100.00\%&0.00&67.03&31.22&$1^{*},3^{*}$&100.00\%&0.00&\textbf{59.53}&39.87&$1^{-},2^{*}$\\
&10&100.00\%&0.00&\textbf{1493.13}&380.08&$2^{*},3^{*}$&100.00\%&0.00&1605.20&472.96&$1^{*},3^{*}$&100.00\%&0.00&1748.20&659.40&$1^{*},2^{*}$\\
&20&100.00\%&0.00&\textbf{26794.30}&7051.89&$2^{-},3^{-}$&99.96\%&0.02&49955.43&244.10&$1^{+},3^{*}$&99.92\%&0.02&50000.00&0.00&$1^{+},2^{*}$\\
&50&99.89\%&0.01&\textbf{125000.00}&0.00&$2^{*},3^{*}$&99.76\%&0.01&\textbf{125000.00}&0.00&$1^{*},3^{*}$&99.72\%&0.01&\textbf{125000.00}&0.00&$1^{*},2^{*}$\\
\hline
\multirow{4}{*}{100}&3&100.00\%&0.00&211.83&83.51&$2^{+},3^{+}$&100.00\%&0.00&109.27&61.33&$1^{-},3^{*}$&100.00\%&0.00&\textbf{107.07}&54.83&$1^{-},2^{*}$\\
&10&100.00\%&0.00&1853.67&411.67&$2^{+},3^{+}$&100.00\%&0.00&1563.40&412.61&$1^{-},3^{*}$&100.00\%&0.00&\textbf{1403.73}&458.76&$1^{-},2^{*}$\\
&20&100.00\%&0.00&\textbf{8377.83}&1259.16&$2^{*},3^{-}$&100.00\%&0.00&9648.03&2079.70&$1^{*},3^{*}$&100.00\%&0.00&9895.17&2165.04&$1^{+},2^{*}$\\
&50&99.94\%&0.00&\textbf{500000.00}&0.00&$2^{*},3^{*}$&99.83\%&0.01&\textbf{500000.00}&0.00&$1^{*},3^{*}$&99.77\%&0.01&\textbf{500000.00}&0.00&$1^{*},2^{*}$\\
\hline
\multirow{4}{*}{200}&3&100.00\%&0.00&451.60&249.54&$2^{+},3^{+}$&100.00\%&0.00&285.97&155.56&$1^{-},3^{*}$&100.00\%&0.00&\textbf{209.80}&79.16&$1^{-},2^{*}$\\
&10&100.00\%&0.00&3161.17&751.84&$2^{+},3^{+}$&100.00\%&0.00&2271.67&624.29&$1^{-},3^{+}$&100.00\%&0.00&\textbf{1769.40}&384.22&$1^{-},2^{-}$\\
&20&100.00\%&0.00&9776.93&1925.45&$2^{+},3^{+}$&100.00\%&0.00&7550.90&1390.00&$1^{-},3^{*}$&100.00\%&0.00&\textbf{7371.20}&1388.57&$1^{-},2^{*}$\\
&50&100.00\%&0.00&\textbf{77706.33}&13011.40&$2^{-},3^{-}$&100.00\%&0.00&106661.57&19402.75&$1^{+},3^{-}$&100.00\%&0.00&132255.53&30508.53&$1^{+},2^{+}$\\
\hline
\multirow{4}{*}{500}&3&100.00\%&0.00&828.90&320.42&$2^{+},3^{+}$&100.00\%&0.00&552.70&276.51&$1^{-},3^{*}$&100.00\%&0.00&\textbf{499.20}&258.36&$1^{-},2^{*}$\\
&10&100.00\%&0.00&7036.17&1475.13&$2^{+},3^{+}$&100.00\%&0.00&4686.20&972.09&$1^{-},3^{*}$&100.00\%&0.00&\textbf{3806.03}&971.33&$1^{-},2^{*}$\\
&20&100.00\%&0.00&19258.20&4942.90&$2^{+},3^{+}$&100.00\%&0.00&13667.33&3053.62&$1^{-},3^{*}$&100.00\%&0.00&\textbf{12294.77}&2767.25&$1^{-},2^{*}$\\
&50&100.00\%&0.00&80004.13&9644.84&$2^{+},3^{+}$&100.00\%&0.00&69792.40&13370.98&$1^{-},3^{*}$&100.00\%&0.00&\textbf{63174.17}&10205.59&$1^{-},2^{*}$\\
\hline
\end{tabular}
\end{scriptsize}
\end{table*}

For a problem instance with graph $G$ and integer $\mu$, assuming there is a population $P^*$ with maximum $gtype$
\[2\sum_{X,Y\in P^*}o_{XY}=\mu(\mu-1)n-\max_P\{gtype(P)\},\]
and maximum edge distances uniformity
\[\max_{X,Y\in P}\{o_{XY}\}-\min_{X,Y\in P}\{o_{XY}\}\leq1,\]
then $P^*\in\argmin_P\{\mathcal{D}(P)\}$. When $G$ is complete and $\mu\leq\left\lfloor\frac{n-1}{2}\right\rfloor$, this assumption is true according to the proof for Theorem~\ref{theo:ed_opt}, and $o_{XY}$ are all zero in $P^*$.

Like the ED approach, this approach determines the survival mechanism of the EA: removing from the population the individual $I\in\argmin_{I\in P}\{\mathcal{D}(P\setminus\{I\})\}$. The same method can be used to derive an efficient implementation of $O(\mu^2n+\mu^2\log\mu)$ time complexity for this mechanism. Consider the fitness function
\[d_P(I)=\text{sort}\left((o_{IX})_{X\in P\setminus\{I\}}\right)\]
with descending sorting order, it can be shown that $\argmin_I\{\mathcal{D}(P\setminus\{I\})\}=\argmax_I\{d_P(I)\}$. As before, we can uniquely define $\mathcal{D}(P)$ and $d_P(I)$ with $(m_i)_{i=1,\dots,n}$ and $(n^I_i)_{i=1,\dots,n}$, respectively. This implies that $\mathcal{D}(P\setminus\{I\})$ is defined by $(m_i-n^I_i)_{i=1,\dots,n}$. For $X,Y\in P$, if $\mathcal{D}(P\setminus\{X\})<\mathcal{D}(P\setminus\{Y\})$ lexicographically, then it must be that there is $j\in[1,n]$ such that $n^X_j>n^Y_j$ and $n^X_i=n^Y_i$ for all $i\in(j,n]$. This means $d_P(X)>d_P(Y)$. Furthermore, $\mathcal{D}(P\setminus\{X\})<\mathcal{D}(P\setminus\{Y\})$ obviously implies $d_P(X)=d_P(Y)$. Intuitively, this implementation removes $I$ that has the smallest edge distance to $P\setminus\{I\}$, defined by the minimum edge distance between $I$ and all tours in $P\setminus\{I\}$. The lexicographic ordering of vectors allows an elegant way to resolve draw cases: comparing the second minimum distances, the third, the fourth and so on. The result is the guaranteed non-increasing edge distance between $I'$ and $P\setminus\{I'\}$ for all other tours $I'$ as well. The consequence is convergence to a solution population, in which each tour is reasonably dissimilar to the rest. This approach aligns closely with the diversity formulated in \cite{DiversityManyObj}, with edge distance being the dissimilarity metric. Since $\mu$ is fixed, we introduce an additional normalization factor.
\[div(P)=\sum_{T\in P}dist(T,P\setminus\{T\})=\frac{1}{\mu n}\sum_{T\in P}\min_{X\in P\setminus\{T\}}\left\lbrace|E(T)\setminus E(X)|\right\rbrace.\]
Note that the difference between $gtype$ and $div$ is the focus on the minimum edge distances, the two measures would be the same if $min$ were replaced by $sum$ or $mean$. We hypothesise that edge distances uniformity is strongly positively correlated to $div$. This is explored further in Section \ref{sec5}.

\section{Experimental investigations}
\label{sec5}

We perform different series of experiments to gain insights into the process of evolving diverse TSP tours. In all experiments our input is a complete graph $G = (V, E)$ with real-valued cost function $w : E \to \mathbb{R}^{+}$. We consider 6 variants of algorithm \ref{alg:ea}, differing in their mutation operators and survival selection mechanisms. The mutation operators are 2-OPT, 3-OPT, 4-OPT, and the survival selection mechanisms are based on the two proposed approaches.

\subsection{Unconstrained Diversity Optimisation}
In this setting our focus is on diverse TSP tours without constraints on tours' qualities. As such, we study our approaches with random initial populations. We experiment with all combinations of $n=\{50,100,200,500\}$ and $\mu=\{3,10,20,50\}$. For each $(n,\mu)$ combination, 30 populations are randomly generated as initial populations for all EA variants, controlling for the initialisation factor. Furthermore, with the established guarantee for optima, a termination criterion is such optima are reached. As for minimising $\mathcal{D}(P)$, another criterion
\[\max_{X,Y\in P}\{o_{XY}\}-\min_{X,Y\in P}\{o_{XY}\}\leq1\]
is added forming a bound on the optima.
An additional limit of $\mu n^2$ iterations is imposed on the experiments. Also, since optimal $gtype$ scores can be calculated, we record the scores in percentages for ease of comparison across all settings.

According to Tables \ref{tb:Edge_counts}, \ref{tb:Pair_dist}, all variants seem to reliably achieve optima in all cases where $\mu\leq\left\lfloor\frac{n-1}{2}\right\rfloor$, except for 4-OPT variants seemingly stuck in local optima when $n=50$ and $\mu=20$. In the other hard cases, none ever reached the optima within the time limits. This suggests that when close to the optima, the probability of increasing $gtype$ in an iteration decreases substantially with increasing $\mu/n$. Also, in such cases, all variants with 2-OPT mutation operator achieve higher $gtype$ score than those with 3-OPT, which in turn outperform those with 4-OPT. Furthermore, for $(n,\mu)=(50,20)$, all 2-OPT variants always reaches the optimum, while 3-OPT variants struggle and 4-OPT variants fail entirely. This indicates that large-step mutation operators are prone to being stuck in local optima when the population is close to maximum $gtype$.

Additionally, minimising $\mathcal{N}(P)$ and minimising $\mathcal{D}(P)$ seem to produce the similar final $gtype$ scores within similar numbers of iterations in all cases, given the same mutation operator used. This indicates that minimising $\mathcal{D}(P)$ also tends to maximise $gtype$ when all tours are accepted. However, in the hard cases, the PD approach achieves somewhat lower $gtype$ scores than ED across all variants. This suggests that the trade-offs between edge diversity and edge distances uniformity are non-trivial near optima.

Another observation is that with small enough $\mu/n$, 2-OPT variants tend to take more iterations than 3-OPT variants, which in turn tend to terminate later than 4-OPT variants. When $\mu/n$ is larger than some number, the trend seems to revert. However, the indication for this observation is weak since no statistically significant difference can be seen in many cases.
\subsection{Constrained Diversity Optimisation}
Now we consider TSP instances from the famous TSPlib, specifically eil51, eil76, eil101. For these experiments, we use the provided optimal tour for each instance, initialise $P$ with $\mu$ copies of it and perform diversity maximisation subject to $c(I) \leq (1 + \alpha)\cdot OPT$ for all $I \in P$. I.~e., we accept tours only if they deviate in length by a factor of at most $(1+\alpha)$ from the optimal tour length OPT. We set up the instances for our experiments with $\mu=\{5,10,20,50\}$ and $\alpha=\{0.05, 0.2, 0.5, 1.0\}$. For each instance, we run each algorithm variant 30 times and record the final population. As before, $gtype$ scores are reported in percentages, since these instances involve complete graphs and Theorem~\ref{theo:ed_opt} applies. In addition, we record $\varsigma(P)=\left(\max_{X,Y\in P}\{o_{XY}\}-\min_{X,Y\in P}\{o_{XY}\}\right)/n$ to observe uniformity in edge distances between tours; lower scores indicate higher uniformity. All scores are averaged over 30 runs.

According to Table~\ref{tb:TSPlib}, $gtype$ scores achieved predictably increase with increasing $\alpha$, albeit with diminishing return. On the other hand, $\varsigma$ scores tend to increase with increasing $\mu$, and dramatically decrease with increasing $\alpha$. In terms of mutation operators, 2-OPT seems to be the best performer overall, while 3-OPT and 4-OPT perform similar. Furthermore, all things equal, minimising $\mathcal{N}(P)$ tends to produce slightly higher $gtype$ scores and noticeably lower $\varsigma$ scores than minimising $\mathcal{D}(P)$. The PD approach seems to better capitalise on increasing $\alpha$, improving edge distances uniformity faster. Moreover, the edge distances uniformity of the PD approach's output populations is less susceptible to compromise due to increasing $\mu$ within this range. These phenomena align with the remark that the ED approach focuses on $gtype$ while the PD approach compromises it for edge distances uniformity.

\begin{table*}[htbp]
\renewcommand{\arraystretch}{1.2845}
\renewcommand\tabcolsep{5.7pt}
\caption{Comparison in terms of diversity (gtype) and pairwise edge distances ranges ($\varsigma$) among all variants of the EA on TSPlib instances. Better values with statistical significance (based on Wilcoxon rank sum tests with 95\% confidence threshold) between ED and PD are bold-faced.
}
\label{tb:TSPlib}
\begin{scriptsize}
\begin{tabular}{crrcccccccccccc}
    \toprule

\multirow{3}{*}{\textbf{}}&\multirow{3}{*}{\textbf{$\mu$}} & \multirow{3}{*}{\textbf{$\alpha$}} &
\multicolumn{6}{c}{\bfseries ED} & \multicolumn{6}{c}{\bfseries PD}\\
\cmidrule(l{2pt}r{2pt}){4-9} \cmidrule(l{2pt}r{2pt}){10-15}
 & & & \multicolumn{2}{c}{\bfseries2-OPT} & \multicolumn{2}{c}{\bfseries3-OPT}& \multicolumn{2}{c}{\bfseries4-OPT}& \multicolumn{2}{c}{\bfseries2-OPT} & \multicolumn{2}{c}{\bfseries3-OPT}& \multicolumn{2}{c}{\bfseries4-OPT} \\
 \cmidrule(l{2pt}r{2pt}){4-5} \cmidrule(l{2pt}r{2pt}){6-7}
 \cmidrule(l{2pt}r{2pt}){8-9} \cmidrule(l{2pt}r{2pt}){10-11}
 \cmidrule(l{2pt}r{2pt}){12-13} \cmidrule(l{2pt}r{2pt}){14-15}
 &&&\textbf{gtype}&\textbf{$\varsigma$}&\textbf{gtype}&\textbf{$\varsigma$}&\textbf{gtype}&\textbf{$\varsigma$}&\textbf{gtype}&\textbf{$\varsigma$}&\textbf{gtype}&\textbf{$\varsigma$}&\textbf{gtype}&\textbf{$\varsigma$}\\
\midrule
\multirow{20}{*}[0.5cm]{\rotatebox[origin=c]{90}{eil51}}&3&0.05&34.27\%&68.37\%&36.23\%&67.78\%&28.95\%&74.05\%&32.07\%&69.54\%&35.95\%&66.14\%&29.78\%&71.96\%\\&
&0.2&70.78\%&31.83\%&67.93\%&35.03\%&63.68\%&40.46\%&71.11\%&30.20\%&65.95\%&35.75\%&63.12\%&38.56\%\\&
&0.5&93.83\%&8.43\%&90.85\%&11.05\%&90.78\%&11.37\%&93.62\%&\textbf{7.12\%}&90.26\%&11.18\%&90.37\%&10.59\%\\&
&1&99.89\%&0.13\%&99.83\%&0.26\%&99.80\%&0.33\%&99.80\%&0.33\%&99.74\%&0.65\%&99.76\%&0.52\%\\\cmidrule(l{2pt}r{2pt}){2-15}
&10&0.05&\textbf{31.82\%}&79.67\%&\textbf{33.64\%}&78.43\%&27.79\%&82.88\%&29.37\%&78.24\%&31.60\%&\textbf{75.95\%}&27.65\%&\textbf{79.61\%}\\&
&0.2&\textbf{63.04\%}&70.39\%&\textbf{60.60\%}&67.12\%&\textbf{57.78\%}&63.40\%&60.99\%&\textbf{44.90\%}&59.23\%&\textbf{47.39\%}&55.99\%&\textbf{51.31\%}\\&
&0.5&\textbf{83.14\%}&30.85\%&\textbf{81.34\%}&36.67\%&\textbf{81.29\%}&37.58\%&82.00\%&\textbf{20.33\%}&80.12\%&\textbf{23.01\%}&79.87\%&\textbf{23.99\%}\\&
&1&\textbf{95.06\%}&10.26\%&\textbf{94.14\%}&11.76\%&\textbf{94.19\%}&11.63\%&94.57\%&\textbf{7.58\%}&93.30\%&\textbf{8.50\%}&93.16\%&\textbf{8.50\%}\\\cmidrule(l{2pt}r{2pt}){2-15}
&20&0.05&\textbf{32.52\%}&95.23\%&\textbf{32.86\%}&91.37\%&\textbf{27.89\%}&90.85\%&29.20\%&\textbf{82.42\%}&31.23\%&\textbf{79.08\%}&26.45\%&\textbf{83.53\%}\\&
&0.2&\textbf{62.32\%}&92.09\%&\textbf{59.98\%}&91.57\%&\textbf{57.46\%}&91.70\%&59.30\%&\textbf{49.15\%}&57.84\%&\textbf{51.37\%}&54.53\%&\textbf{55.29\%}\\&
&0.5&\textbf{80.95\%}&59.02\%&\textbf{79.24\%}&63.27\%&\textbf{79.60\%}&64.84\%&79.03\%&\textbf{25.29\%}&77.69\%&\textbf{27.84\%}&77.32\%&\textbf{29.08\%}\\&
&1&\textbf{92.17\%}&19.87\%&\textbf{91.41\%}&22.42\%&\textbf{91.58\%}&21.31\%&90.97\%&\textbf{11.83\%}&90.28\%&\textbf{12.42\%}&90.24\%&\textbf{12.75\%}\\\cmidrule(l{2pt}r{2pt}){2-15}
&50&0.05&\textbf{32.66\%}&100.00\%&\textbf{33.23\%}&100.00\%&\textbf{28.01\%}&100.00\%&29.81\%&\textbf{85.88\%}&31.75\%&\textbf{82.48\%}&26.93\%&\textbf{85.95\%}\\&
&0.2&\textbf{63.26\%}&99.48\%&\textbf{61.52\%}&97.78\%&\textbf{58.92\%}&96.99\%&59.68\%&\textbf{53.59\%}&58.36\%&\textbf{57.12\%}&55.04\%&\textbf{59.80\%}\\&
&0.5&\textbf{80.74\%}&92.35\%&\textbf{79.56\%}&90.59\%&\textbf{79.80\%}&88.76\%&79.15\%&\textbf{29.67\%}&78.12\%&\textbf{32.29\%}&77.94\%&\textbf{33.66\%}\\&
&1&\textbf{91.65\%}&48.95\%&\textbf{91.25\%}&52.68\%&\textbf{91.33\%}&48.37\%&90.60\%&\textbf{15.69\%}&90.42\%&\textbf{15.75\%}&90.42\%&\textbf{16.67\%}\\\hline
\multirow{20}{*}[0.5cm]{\rotatebox[origin=c]{90}{eil76}}&3&0.05&30.12\%&71.84\%&30.06\%&72.89\%&24.20\%&78.60\%&29.04\%&71.97\%&29.80\%&71.89\%&24.08\%&76.97\%\\&
&0.2&70.23\%&31.67\%&63.98\%&38.90\%&60.39\%&42.02\%&69.05\%&32.15\%&63.17\%&37.76\%&59.11\%&42.06\%\\&
&0.5&94.66\%&6.49\%&90.63\%&11.45\%&91.05\%&11.14\%&94.72\%&5.88\%&90.82\%&\textbf{9.74\%}&90.29\%&10.61\%\\&
&1&99.91\%&0.22\%&99.63\%&0.57\%&99.80\%&0.48\%&99.90\%&0.09\%&99.65\%&0.53\%&99.69\%&0.57\%\\\cmidrule(l{2pt}r{2pt}){2-15}
&10&0.05&\textbf{28.52\%}&79.39\%&\textbf{29.33\%}&78.95\%&23.14\%&83.60\%&27.09\%&78.73\%&28.35\%&79.17\%&22.87\%&82.59\%\\&
&0.2&\textbf{61.56\%}&57.19\%&\textbf{58.62\%}&56.75\%&\textbf{56.01\%}&60.13\%&59.57\%&\textbf{44.91\%}&56.41\%&\textbf{49.08\%}&53.67\%&\textbf{52.24\%}\\&
&0.5&\textbf{84.71\%}&28.60\%&\textbf{81.03\%}&36.45\%&\textbf{81.92\%}&33.64\%&82.94\%&\textbf{18.99\%}&79.54\%&\textbf{23.07\%}&79.45\%&\textbf{23.90\%}\\&
&1&\textbf{95.92\%}&7.94\%&\textbf{94.39\%}&9.96\%&\textbf{94.78\%}&9.43\%&95.32\%&\textbf{5.83\%}&93.66\%&\textbf{7.72\%}&93.74\%&\textbf{7.72\%}\\\cmidrule(l{2pt}r{2pt}){2-15}
&20&0.05&\textbf{28.74\%}&88.95\%&\textbf{28.85\%}&88.68\%&\textbf{23.69\%}&88.42\%&26.96\%&\textbf{81.49\%}&27.38\%&\textbf{81.97\%}&22.25\%&\textbf{85.18\%}\\&
&0.2&\textbf{61.19\%}&92.76\%&\textbf{58.37\%}&92.89\%&\textbf{56.06\%}&91.45\%&58.50\%&\textbf{48.29\%}&55.41\%&\textbf{53.16\%}&52.20\%&\textbf{56.67\%}\\&
&0.5&\textbf{82.35\%}&53.90\%&\textbf{79.64\%}&65.39\%&\textbf{80.32\%}&58.68\%&79.72\%&\textbf{23.73\%}&77.58\%&\textbf{26.84\%}&77.40\%&\textbf{28.42\%}\\&
&1&\textbf{93.03\%}&16.10\%&\textbf{91.95\%}&20.04\%&\textbf{92.37\%}&18.11\%&91.67\%&\textbf{10.26\%}&90.67\%&\textbf{11.58\%}&90.78\%&\textbf{11.84\%}\\\cmidrule(l{2pt}r{2pt}){2-15}
&50&0.05&\textbf{28.84\%}&100.00\%&\textbf{28.41\%}&99.91\%&\textbf{23.99\%}&98.95\%&27.01\%&\textbf{84.47\%}&27.77\%&\textbf{85.09\%}&22.52\%&\textbf{87.28\%}\\&
&0.2&\textbf{62.42\%}&98.95\%&\textbf{59.32\%}&98.73\%&\textbf{56.80\%}&98.07\%&58.16\%&\textbf{52.54\%}&55.73\%&\textbf{56.80\%}&52.29\%&\textbf{60.35\%}\\&
&0.5&\textbf{81.53\%}&86.62\%&\textbf{79.74\%}&90.39\%&\textbf{80.35\%}&88.03\%&79.05\%&\textbf{27.50\%}&77.21\%&\textbf{31.01\%}&77.08\%&\textbf{32.63\%}\\&
&1&\textbf{91.73\%}&42.19\%&\textbf{91.18\%}&46.62\%&\textbf{91.39\%}&38.55\%&90.32\%&\textbf{13.68\%}&89.84\%&\textbf{14.78\%}&90.02\%&\textbf{14.91\%}\\\hline
\multirow{20}{*}[0.5cm]{\rotatebox[origin=c]{90}{eil101}}&3&0.05&35.84\%&65.97\%&\textbf{36.79\%}&65.74\%&31.21\%&70.99\%&35.69\%&65.71\%&35.49\%&65.51\%&30.24\%&70.69\%\\&
&0.2&\textbf{72.61\%}&29.60\%&67.50\%&34.69\%&65.30\%&36.40\%&71.29\%&29.11\%&65.56\%&35.08\%&63.99\%&36.83\%\\&
&0.5&95.03\%&5.91\%&91.86\%&9.64\%&92.29\%&9.27\%&94.72\%&5.78\%&91.60\%&8.98\%&91.63\%&9.11\%\\&
&1&99.79\%&0.40\%&99.49\%&0.76\%&99.72\%&0.50\%&99.81\%&0.36\%&99.57\%&0.66\%&99.85\%&\textbf{0.23\%}\\\cmidrule(l{2pt}r{2pt}){2-15}
&10&0.05&31.88\%&76.27\%&\textbf{33.44\%}&74.19\%&\textbf{28.88\%}&78.88\%&31.36\%&\textbf{73.37\%}&32.60\%&\textbf{72.05\%}&27.77\%&\textbf{76.93\%}\\&
&0.2&\textbf{63.70\%}&49.64\%&\textbf{60.78\%}&54.16\%&\textbf{58.69\%}&55.38\%&62.44\%&\textbf{40.89\%}&59.29\%&\textbf{44.95\%}&56.49\%&\textbf{47.95\%}\\&
&0.5&\textbf{85.14\%}&25.78\%&\textbf{82.55\%}&30.43\%&\textbf{82.75\%}&33.00\%&83.84\%&\textbf{17.52\%}&81.64\%&\textbf{20.20\%}&80.78\%&\textbf{21.55\%}\\&
&1&\textbf{96.29\%}&6.80\%&\textbf{95.07\%}&9.08\%&\textbf{95.25\%}&8.75\%&95.97\%&\textbf{4.95\%}&94.37\%&\textbf{6.77\%}&94.37\%&\textbf{6.86\%}\\\cmidrule(l{2pt}r{2pt}){2-15}
&20&0.05&\textbf{32.80\%}&85.28\%&\textbf{33.52\%}&85.31\%&\textbf{28.74\%}&85.97\%&30.58\%&\textbf{76.63\%}&31.78\%&\textbf{75.38\%}&27.21\%&\textbf{79.90\%}\\&
&0.2&\textbf{63.37\%}&92.48\%&\textbf{60.27\%}&89.70\%&\textbf{58.75\%}&89.90\%&60.83\%&\textbf{44.36\%}&58.11\%&\textbf{48.91\%}&55.30\%&\textbf{51.16\%}\\&
&0.5&\textbf{82.93\%}&56.17\%&\textbf{80.68\%}&57.49\%&\textbf{81.19\%}&55.68\%&81.04\%&\textbf{21.22\%}&79.08\%&\textbf{24.19\%}&79.05\%&\textbf{25.21\%}\\&
&1&\textbf{93.71\%}&14.98\%&\textbf{92.62\%}&17.99\%&\textbf{93.06\%}&15.91\%&92.28\%&\textbf{9.01\%}&91.31\%&\textbf{10.23\%}&91.42\%&\textbf{10.43\%}\\\cmidrule(l{2pt}r{2pt}){2-15}
&50&0.05&\textbf{34.32\%}&100.00\%&\textbf{33.17\%}&99.47\%&\textbf{28.96\%}&99.01\%&30.26\%&\textbf{79.44\%}&31.34\%&\textbf{78.75\%}&26.91\%&\textbf{82.71\%}\\&
&0.2&\textbf{63.92\%}&98.48\%&\textbf{60.95\%}&97.92\%&\textbf{59.39\%}&97.66\%&59.76\%&\textbf{48.25\%}&57.89\%&\textbf{52.24\%}&54.86\%&\textbf{55.74\%}\\&
&0.5&\textbf{81.65\%}&86.30\%&\textbf{79.96\%}&90.23\%&\textbf{80.51\%}&85.87\%&79.34\%&\textbf{24.85\%}&78.20\%&\textbf{27.36\%}&77.85\%&\textbf{28.84\%}\\&
&1&\textbf{91.58\%}&34.69\%&\textbf{90.88\%}&45.38\%&\textbf{91.20\%}&36.60\%&90.01\%&\textbf{12.34\%}&89.60\%&\textbf{13.17\%}&89.80\%&\textbf{13.43\%}\\\hline
\end{tabular}
\end{scriptsize}
\end{table*}

\begin{figure}[t]
\centering
\begin{subfigure}{0.4\linewidth}
\centering
\includegraphics[width=1\linewidth]{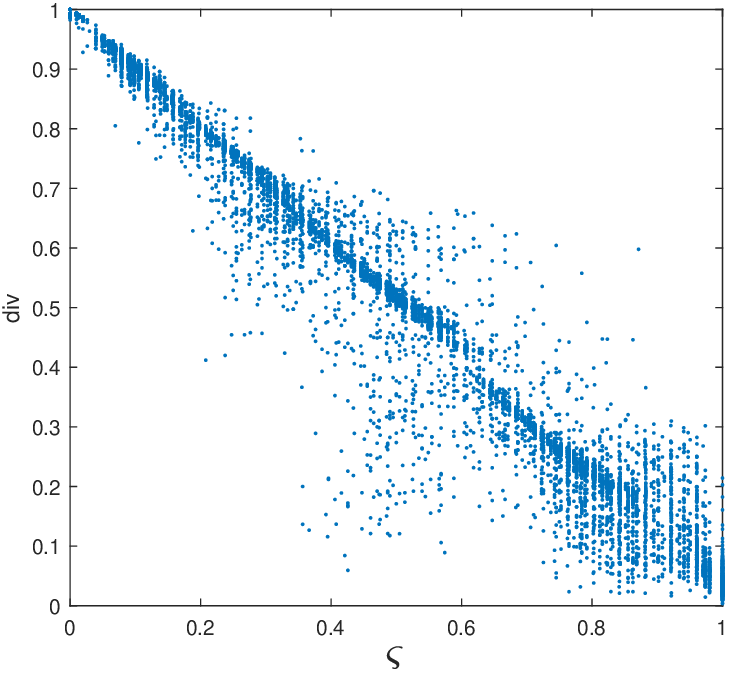}
\caption{$\varsigma$ and $div$}\label{fig:uniform_div_corr}
\end{subfigure}
\begin{subfigure}{0.4\linewidth}
\centering
\includegraphics[width=1\linewidth]{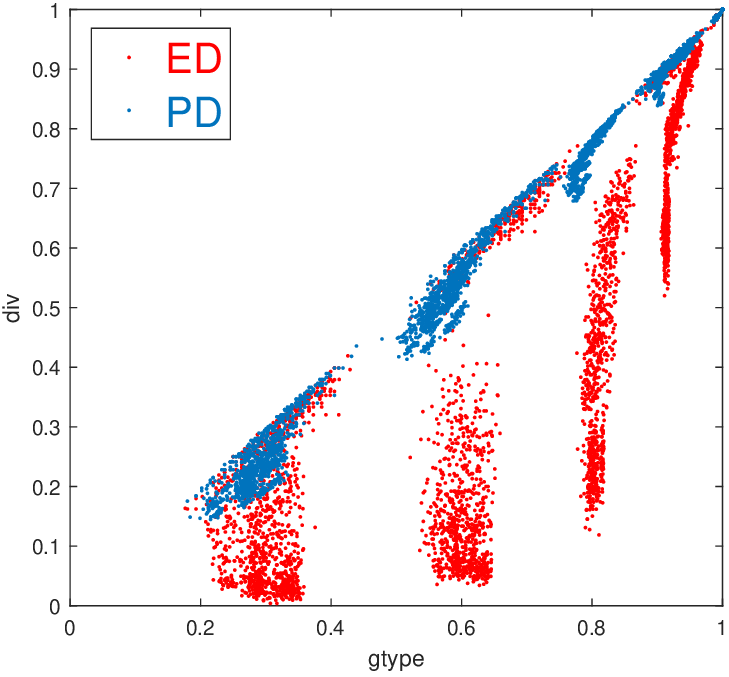}
\caption{$gtype$ and $div$}\label{fig:gtype_div}
\end{subfigure}
\caption{
Scatter plots of all 8640 runs on the TSPlib instances. Each point corresponds to a final population after one run. The Pearson correlation coefficient between $\varsigma$ and $div$ is $-0.9815$ with $p<0.0001$.}
\end{figure}
Additionally, we investigate the correlation between edge distances uniformity and the diversity score $div(P)$. Figure \ref{fig:uniform_div_corr} shows a strong negative linear correlation between $div(P)$ and $\varsigma(P)$ across all cases. This suggests that focusing on edge distances uniformity, while maximising $gtype$, is effective in maximising $div$. Combined with earlier observations, we can conclude that the PD approach is much more likely to make a better compromise between maximising $gtype$ and maximising $div$. This is illustrated in Figure~\ref{fig:gtype_div}, indicating that for each $\alpha$ value, the PD approach maintains higher $div$ scores across all $\mu$ values without significantly sacrificing $gtype$ scores.

\begin{figure}[t]
\centering
\begin{subfigure}{0.24\linewidth}
\centering
\includegraphics[width=1\linewidth]{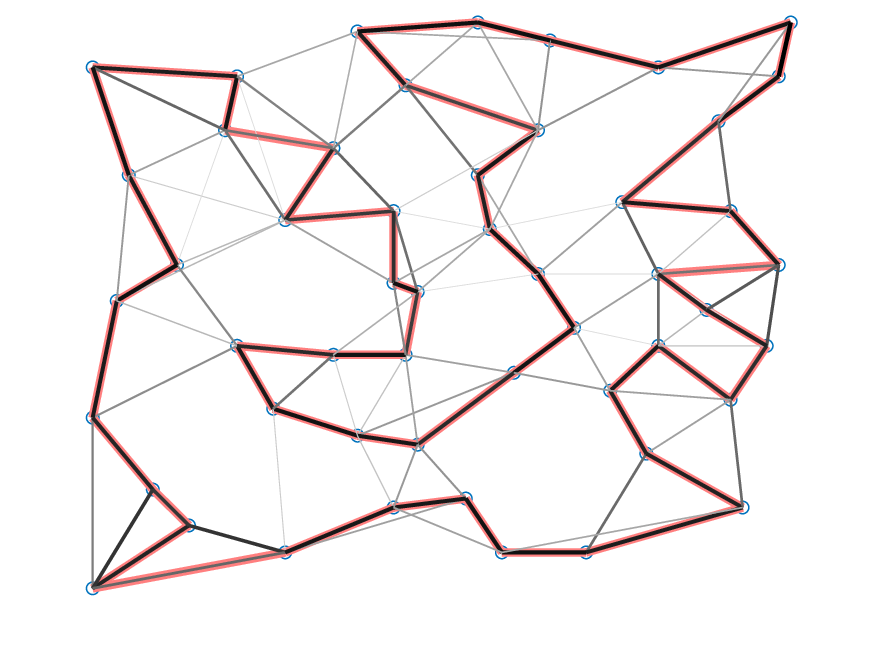}
\caption{ED, $\alpha=0.05$}
\end{subfigure}
\begin{subfigure}{0.24\linewidth}
\centering
\includegraphics[width=1\linewidth]{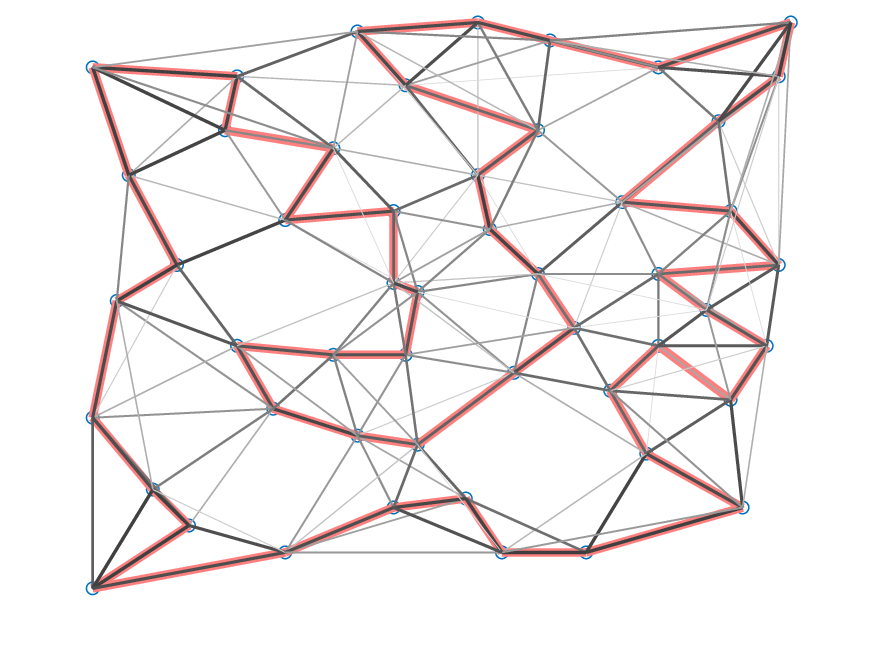}
\caption{ED, $\alpha=0.2$}
\end{subfigure}
\begin{subfigure}{0.24\linewidth}
\centering
\includegraphics[width=1\linewidth]{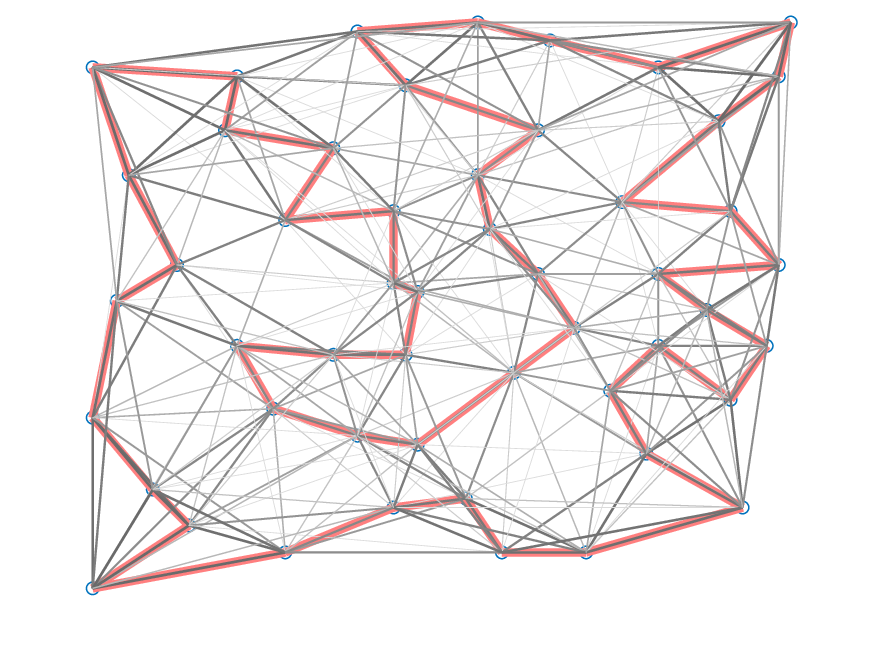}
\caption{ED, $\alpha=0.5$}
\end{subfigure}
\begin{subfigure}{0.24\linewidth}
\centering
\includegraphics[width=1\linewidth]{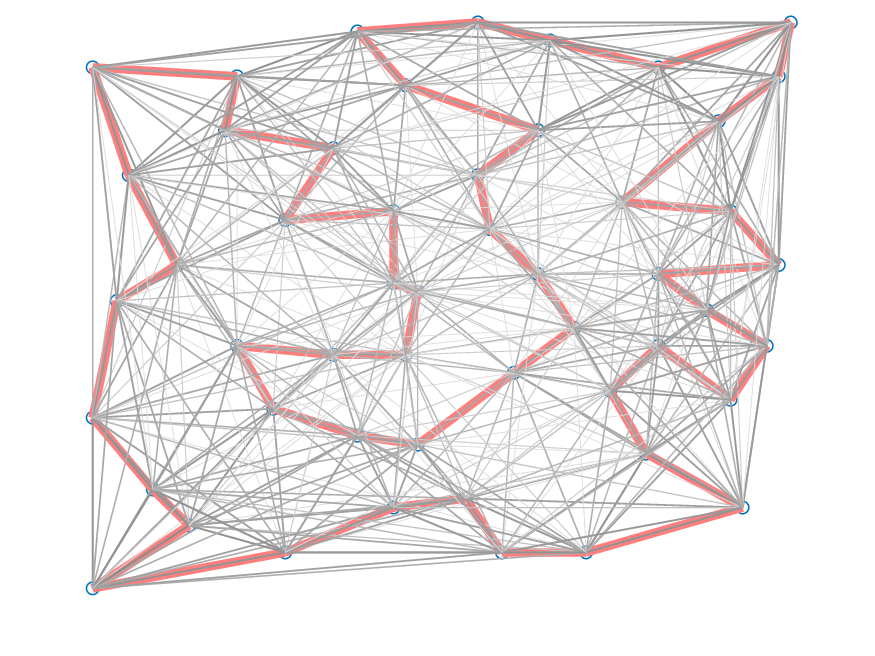}
\caption{ED, $\alpha=1.0$}
\end{subfigure}
\begin{subfigure}{0.24\linewidth}
\centering
\includegraphics[width=1\linewidth]{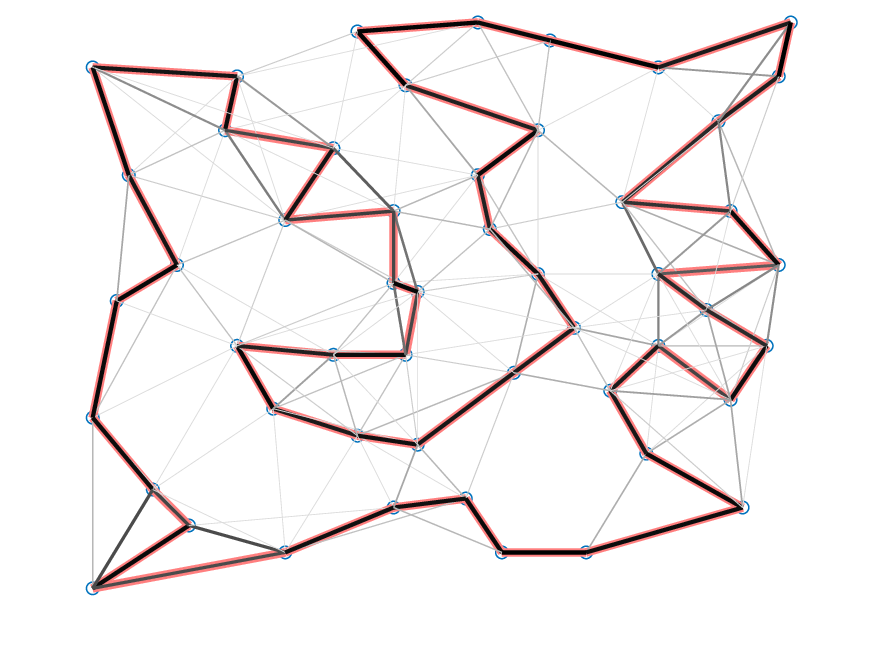}
\caption{PD, $\alpha=0.05$}
\end{subfigure}
\begin{subfigure}{0.24\linewidth}
\centering
\includegraphics[width=1\linewidth]{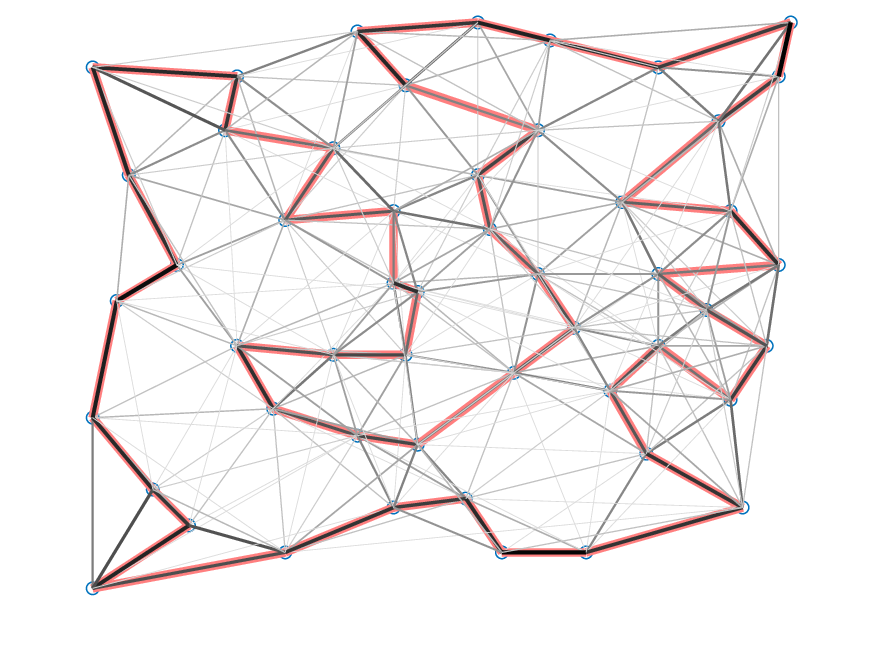}
\caption{PD, $\alpha=0.2$}
\end{subfigure}
\begin{subfigure}{0.24\linewidth}
\centering
\includegraphics[width=1\linewidth]{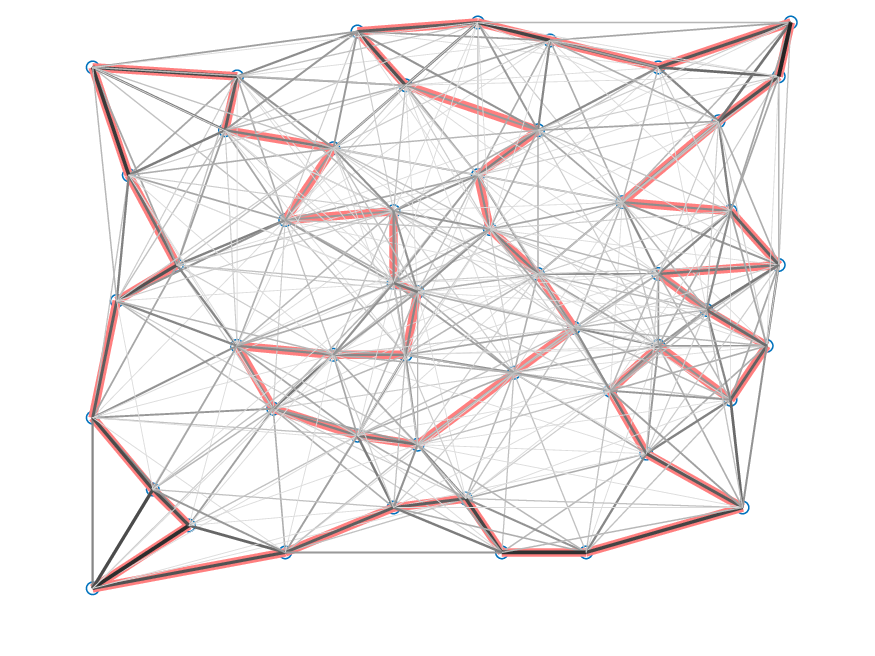}
\caption{PD, $\alpha=0.5$}
\end{subfigure}
\begin{subfigure}{0.24\linewidth}
\centering
\includegraphics[width=1\linewidth]{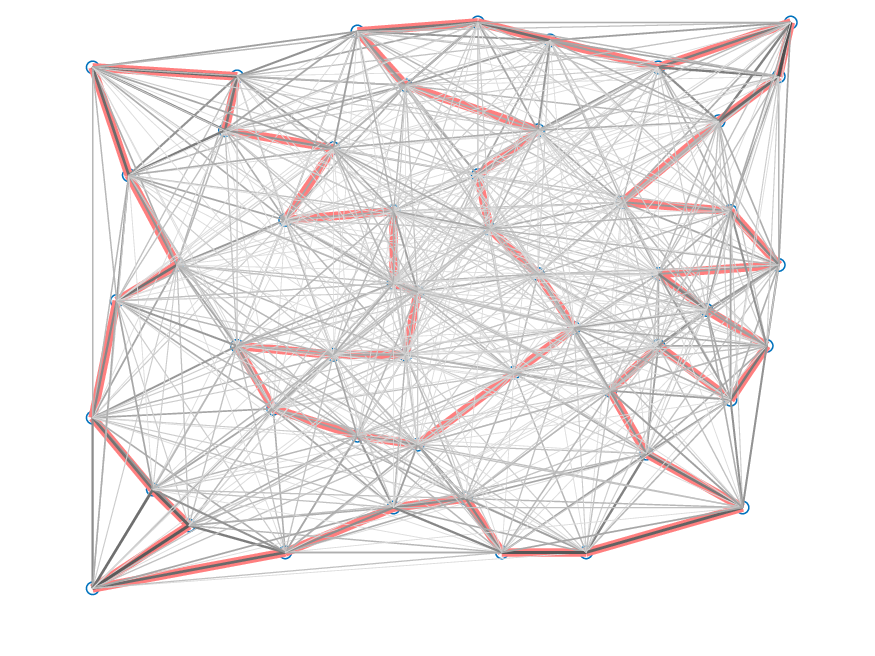}
\caption{PD, $\alpha=1.0$}
\end{subfigure}
\caption{Visualised edge counts from resulted populations in eil51 cases with $\mu=50$. The optimal tour is marked with red edges. Darker edges have higher counts.}
\label{fig:edge_count}
\end{figure}

\begin{figure*}[htbp]
\centering
\begin{subfigure}{1\linewidth}
\centering
\includegraphics[width=1\linewidth]{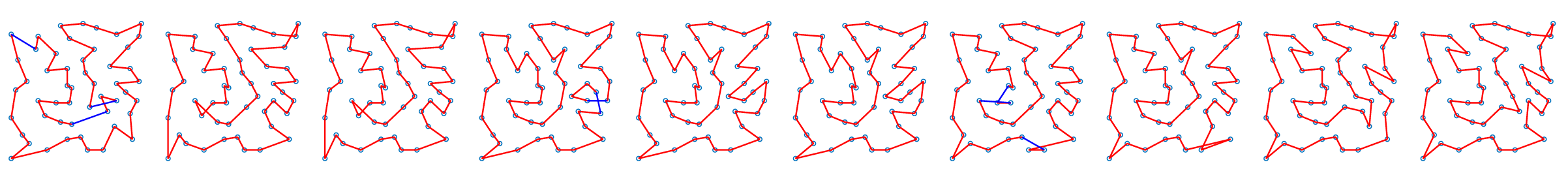}
\caption{ED, $\alpha=5\%$}
\end{subfigure}
~\\
\begin{subfigure}{1\linewidth}
\centering
\includegraphics[width=\linewidth]{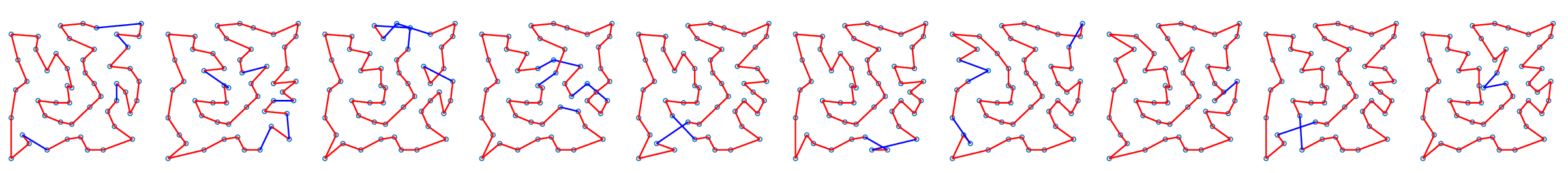}
\caption{PD, $\alpha=5\%$}
\end{subfigure}
~\\
\begin{subfigure}{1\linewidth}
\centering
\includegraphics[width=1\linewidth]{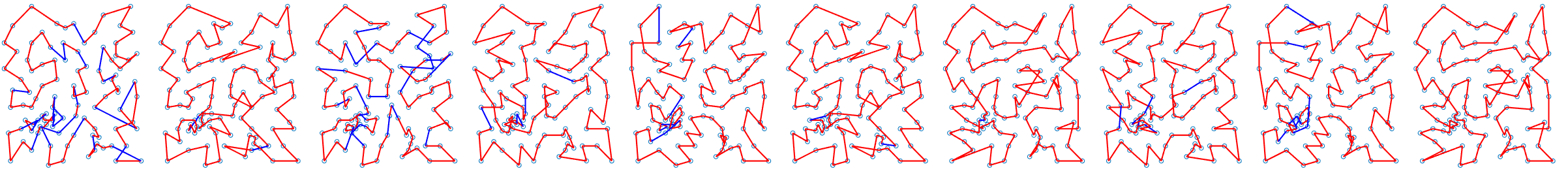}
\caption{ED, $\alpha=20\%$}
\end{subfigure}
~\\
\begin{subfigure}{1\linewidth}
\centering
\includegraphics[width=1\linewidth]{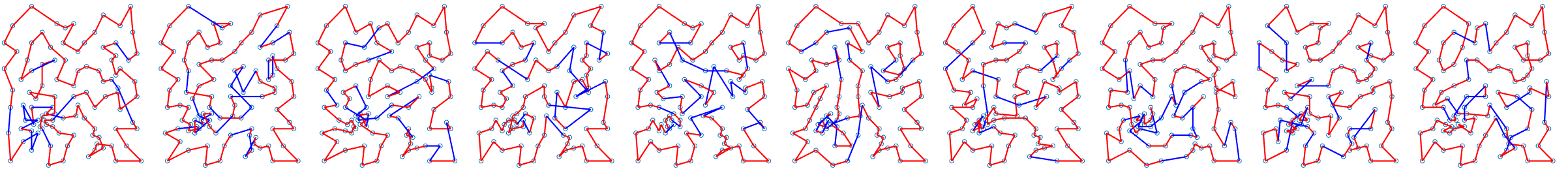}
\caption{PD, $\alpha=20\%$}
\end{subfigure}
\caption{
Visualised tour populations from resulted populations in eil51 cases with $\mu=10$ and 2-OPT as mutation operator. Red edges are shared by at least two tours in the population, and blue ones are unique to the tour. Left-to-right is the ascending tour length order.}
\label{fig:populations}
\end{figure*}

The visuals in Figure~\ref{fig:edge_count} show that the PD approach results in fewer zero-count edges than the ED, regardless of $\alpha$. It also results in higher maximum edge counts than the ED in those cases. This is because minimising $\mathcal{N}(P)$ flattens the edge counts distribution from the top down with the descending sorting order. The implication is that the counts distribution among the lower end is not guaranteed improvement, especially among the zero-count. On the other hand, minimising $\mathcal{D}(P)$, while not directly dealing with edge counts, tends to equalise this distribution while relaxing minimising higher count edges. Consequently, higher maximum edge counts are achieved, but fewer edges have high counts and more edges have low non-zero counts. Consequently, individuals are more likely to contain more unique edges (edges with count 1), which is also a mark of highly diverse populations. As Figure~\ref{fig:populations} shows, with small $\alpha$, the populations generated by the ED approach tend to contain duplicated tours and tours without unique edges. In contrast, tours produced by the PD approach tend to exhibit more uniqueness and stand out from the rest.

\section{Conclusion}
Evolutionary diversity optimisation aims to generate a set of diverse solutions where all solutions meet given quality criteria. We have introduced and examined for the first time evolutionary diversity optimisation for a classical combinatorial optimisation problems. We introduced two diversity measures that can be used for the Traveling Salesperson Problem and evaluated their performance when used for simple population-based elitist evolutionary algorithms. The results show that both measures can be optimized well in the unconstrained case where all tours meet the quality criterion. Furthermore, our investigations for TSPlib instances point out the increase in terms of diversity that can be obtained when relaxing the quality constraint determined by the required approximation ratio. We also highlighted some differences between populations generated by these two approaches.

The focus of this paper has been on the introduced diversity measures for diversity optimisation of the TSP and their performance in simple population-based elitist evolutionary algorithms.
For future research, it would be interesting to incorporate these measures into state of the art evolutionary algorithms for the TSP and evaluate their performance.

\section*{Acknowledgment}
This work has been supported by the Australian Research Council (ARC) through grant DP190103894.

\bibliographystyle{unsrt}
\bibliography{bib,references}


\end{document}